\documentclass{article}
\usepackage[preprint]{neurips_2025}

\usepackage[utf8]{inputenc} % allow utf-8 input
\usepackage[T1]{fontenc}    % use 8-bit T1 fonts
\usepackage{hyperref}       % hyperlinks
\usepackage{url}            % simple URL typesetting
\usepackage{booktabs}       % professional-quality tables
\usepackage{amsfonts}       % blackboard math symbols
\usepackage{amsmath}        % math packages
\usepackage{amssymb}        % math packages
\usepackage{amsthm}         % math packages
\usepackage{nicefrac}       % compact symbols for 1/2, etc.
\usepackage{microtype}      % microtypography
\usepackage{xcolor}         % colors
\usepackage[pdftex]{graphicx}
\usepackage{subcaption}
\usepackage{algpseudocode}
\usepackage{algorithm}

\newtheorem{lemma}{Lemma}
\newtheorem{theorem}{Theorem}

\newtheorem{definition}{Definition}

\DeclareMathOperator*{\argmax}{arg\,max}
\newcommand{\x}{\mathbf{x}}
\newcommand{\y}{\mathbf{y}}
\newcommand{\z}{\mathbf{z}}
\newcommand{\gt}{\boldsymbol{\theta}^*}
\newcommand{\mle}{\hat{\boldsymbol{\theta}}}
\newcommand{\error}{\gt - \mle_t}
\newcommand{\util}[1]{\langle #1, \gt \rangle}
\newcommand{\link}{F}
\newcommand{\X}{\mathcal{X}}
\renewcommand{\P}{\mathbb{P}}
\newcommand{\E}{\mathbb{E}}
\newcommand{\R}{\mathbb{R}}
\renewcommand{\H}{\mathcal{H}}
\newcommand{\D}{\mathcal{D}}
\newcommand{\dist}{\mathcal{P}}
\newcommand{\norm}[1]{\Vert #1 \Vert_2}
\newcommand{\normbig}[1]{\left\Vert #1 \right\Vert_2}
\newcommand{\algname}{\texttt{ROAM}}
\newcommand{\colstim}{\texttt{CoLSTIM}}
\newcommand{\utilml}[1]{\langle #1, \mle_t \rangle}
\newcommand{\normvt}[1]{\Vert #1 \Vert_{V_t^{-1}}}
\newcommand{\normA}[1]{\Vert #1 \Vert_{A}}
\newcommand{\pd}{\mathcal{S}_+}
\PassOptionsToPackage{numbers}{natbib}

\title{Recycling History: Efficient Recommendations\\from Contextual Dueling Bandits}

\author{%
  Suryanarayana Sankagiri \\
  School of Communication and Computer Science\\
  EPFL, Switzerland \\
  \texttt{suryanarayana.sankagiri@epfl.ch} \\
  \And
  Jalal Etesami \\
  Department of Computer Science \\
  TU Munich, Germany \\
  \texttt{j.etesami@tum.de} \\
  \AND
  Pouria Fatemi \\
  Department of Mathematics \\
  TU Munich, Germany \\
  \texttt{pouria.fatemi@tum.de} \\
  \And
  Matthias Grossglauser \\
  School of Communication and Computer Science\\
  EPFL, Switzerland \\
  \texttt{matthias.grossglauser@epfl.ch} \\
}

\begin{document}

\maketitle

\begin{abstract}
    The contextual duelling bandit problem models adaptive recommender systems, where the algorithm presents a set of items to the user, and the user’s choice reveals their preference. This setup is well suited for implicit choices users make when navigating a content platform, but does not capture other possible comparison queries. Motivated by the fact that users provide more reliable feedback after consuming items, we propose a new bandit model that can be described as follows. The algorithm recommends one item per time step; after consuming that item, the user is asked to compare it with another item chosen from the user's consumption history. Importantly, in our model, this comparison item can be chosen without incurring any additional regret, potentially leading to better performance. However, the regret analysis is challenging because of the temporal dependency in the user's history. To overcome this challenge, we first show that the algorithm can construct informative queries provided the history is rich, \emph{i.e.}, satisfies a certain diversity condition. We then show that a short initial random exploration phase is sufficient for the algorithm to accumulate a rich history with high probability. This result, proven via matrix concentration bounds, yields $O(\sqrt{T})$ regret guarantees. Additionally, our simulations show that reusing past items for comparisons can lead to significantly lower regret than only comparing between simultaneously recommended items.
\end{abstract}

\section{Introduction}\label{sec:intro}

Recommender systems are central to digital platforms, helping users navigate the vast set of options by filtering items based on their preferences \cite{bobadilla2013recommender}. A user's taste profile is typically learned from the feedback they provide. One fundamental question in the design of recommender systems is how to elicit feedback effectively. A basic distinction can be drawn between {\em implicit}  and {\em explicit} feedback. The former is obtained by observing user actions (deleting vs saving in a wishlist) and the latter through explicit prompts (rating an item from one to five stars). In contrast to implicit modes, explicit feedback is typically obtained after the user has consumed the items. Therefore, they provide more accurate reflection of preferences. Another dimension of elicitation design is whether to seek {\em ordinal} (ranking- or comparison-based) or  {\em cardinal} (rating-based) feedback \cite{wang2019your}. While ratings are the most common form of feedback, seeking comparisons instead could offer several advantages. For one, it naturally eliminates the effect of dynamically varying user biases, arising due to mood, experiences, etc. Equally importantly, it gets around the discretization problem: it is difficult to distinguish among many items with five stars, but a comparison between two such items still extracts useful information.

\paragraph{Model Contribution.} Motivated by the above considerations, we aim to formulate a recommender system that learns user tastes through explicit comparisons among consumed items. We adopt the classical contextual bandit (CB) framework for this task. Specifically, we assume there is a single user who is repeatedly served recommendations by the bandit algorithm. The user and the items are endowed with feature vectors, representing their tastes and characteristics respectively. While the user's feature $\gt$ is unknown, the items' features $\x$ are assumed to be known and fixed. At every time step $t$, the algorithm is provided a set of items $\X_t$ (potentially the result of a search query). The algorithm picks a single item, $\x_t$, from this set and recommends it to the user. After the user has consumed the item, the system asks the user to compare it with another item $\y_t$ {\em consumed by the user in the past}. That is, $\y_t$ must belong to $\H_t$, the set of items consumed by the user in the past. Based on the user's response to this comparison query, the algorithm updates its estimate of $\gt$ and refines its future recommendations. The algorithm's performance is measured in terms of the regret, {\em i.e.}, the suboptimality in utility, of the recommended items $\x_t$.

The contextual dueling bandits (CDB) problem, introduced by \cite{saha2021optimal} and later studied by \cite{bengs2022stochastic}, shares several features with our setup. However, there are also important differences, which we highlight below. The model adopted by \cite{saha2021optimal, bengs2022stochastic} assumes that at each time $t$, two fresh items  $(\x_t, \y_t)$ are drawn from $\X_t$ and presented to the user, from which the user picks one. The quality of the recommendation is measured in terms of the average regret over both $\x_t$ and $\y_t$. In contrast, our model assumes that $\x_t$ is chosen from $\X_t$ and $\y_t$ from $\H_t$ (the user's consumption history). Regret is measured only in terms of $\x_t$, because the regret for $\y_t$ has already been accounted for when it was consumed. We refer to our model as the {\em history-constrained CDB model} and to the model in the literature as the {\em concurrent CDB model}. 

Practically speaking, the concurrent CDB models a scenario where a user is offered two options and asked to pick a movie to watch, while the history-constrained model asks the user whether they prefer the movie they have just watched over one that they watched in the past. 
The concurrent model is well suited for learning from the \textit{implicit} comparisons that users make at the theatre or on a streaming platform (\cite{saha2021optimal} motivates their model through a similar example). Although, technically, the model does not preclude the possibility of seeking explicit comparisons, it is often impractical, or wasteful, to force the user to consume two items just to elicit a comparison. The history-constrained model overcomes this limitation of the concurrent model, providing a framework to seek \textit{explicit comparisons among consumed items}.

This seemingly minor difference in formulation has interesting consequences in terms of the algorithm design as well as regret bounds. On the one hand, simultaneously choosing $\x_t$ and $\y_t$ from the same set $\X_t$ makes it easier to formulate informative queries in the concurrent model. The asymmetry introduced in our model complicates this key step, which makes it challenging to bound the regret. In particular, the history set $\H_t$ cannot be designed knowing $\x_t$. On the other hand, the history-constrained model allows the user additional flexibility in navigating the exploration-exploitation trade-off. To elaborate, because the choice of $\y_t$ doesn't incur any regret, it can be chosen purely with an exploration objective. In contrast, optimal algorithms for the concurrent model, such as the ones in \cite{saha2021optimal, bengs2022stochastic} must delicately balance both factors when selecting $(\x_t, \y_t)$, as the regret depends on both items. Thus, a well-designed algorithm for the history-constrained model may obtain better regret than what is possible in the concurrent model.

\paragraph{Algorithmic Contribution.} In this work, we present an algorithm called \algname ~(Regret Once, Ask Many), that provably achieves $O(\sqrt{T})$ cumulative regret under the history-constrained model. 
To elaborate, while choosing what query to ask the user, the algorithm optimizes a metric quantifying the uncertainty in the estimate $\mle$. Intuitively, the algorithm needs to probe the user along different dimensions (axes), in order to narrow down on the user's tastes. Thus, there needs to be sufficient diversity in the choice of comparison queries, so that the algorithm can frame a query probing any arbitrary direction. We make this notion precise by defining a property called {\em rich history}. Once the history is rich, the algorithm can pick $\y_t$ among these items at every step $t$ in order to learn about the user's tastes. The algorithm's name reflects this key feature.

Accumulating a rich history requires exploration (essentially, recommending items at random) and therefore incurs regret. For optimal regret bounds, it is important to ensure that this period of exploration remains small. We use matrix concentration bounds to show that a short exploration phase lasting $\Tilde{O}(d)$ steps is sufficient for the algorithm to accumulate a rich history with high probability. Notably, this degree of initial exploration is typically used by CB algorithms \cite{saha2021optimal, bengs2022stochastic} to get a reasonable estimate $\mle$ to start with. Using this result, and following the proof steps of \cite{saha2021optimal},  we obtain $O(\sqrt{T})$ regret guarantees. Although our theoretical regret bounds may have suboptimal coefficients on the $O(\sqrt{T})$, we find through simulations on synthetic data that the regret suffered by \algname~ scale gracefully with the model parameters, such as dimension. In particular, with appropriate parameters for a fair comparison, we demonstrate that \algname~ (applied to the history-constrained setting) has significantly lower regret than \colstim, the state-of-the art algorithm for the concurrent setting. This demonstrates that the flexibility of recycling items from the user's consumption history to elicit comparisons has tangible benefits.

\paragraph{Related Work.}  As mentioned above, our model shares many similarities with the concurrent CDB model proposed by \cite{saha2021optimal} and followed upon by \cite{bengs2022stochastic}. This model, in turn, is a natural merger of two classical models: the stochastic linear bandit model \cite{abbasi2011improved} and the $K$-armed duelling bandit model \cite{yue2012k}. (See \cite{lattimore2020bandit} for an introduction to linear bandits and \cite{bengs2021preference} for a review of duelling bandits). In particular, the proofs in \cite{saha2021optimal, bengs2022stochastic}, as well as our work, rely on some key technical lemmas proven for generalised linear bandits \cite{li2017provably}. We note that \cite{dudik2015contextual} and \cite{saha2022efficient} also deal with contextual duelling bandits, but their model has a completely different interpretation. Instead of referring to item features, context here refers to a global variable set arbitrarily, which influences the choice probabilities. 

Although linear and generalized linear bandits have been widely used in practical recommender systems \cite{pereira2019online, bendada2020carousel, he2020contextual}, the same cannot be said for contextual duelling bandits (CDBs). In fact, to the best of our knowledge, there are no papers that test CDB algorithms on real data. (One instance of testing a CB algorithm on preference data is presented in \cite{agnihotri2024online}, but their method does not fit the CDB model). Moreover, the prior work on concurrent CDBs \cite{saha2021optimal, bengs2022stochastic} have a limited discussion on the potential of this model in recommender systems. 
This is perhaps in part because learning from comparisons (instead of ratings) is, to this date, not widely accepted in the recommender systems community. We hope that our theoretical modeling and analysis, combined with the empirically demonstrated efficacy of comparison-based learning in the context of LLM alignment \cite{ouyang2022training, rafailov2023direct}, leads to the design of comparison-based recommender systems in the near future.

\section{Model and Algorithm}\label{sec:model}

\paragraph{Basic Setup.}
We begin by stating the modelling assumptions of our history-constrained CDB framework. 
The model describes the sequential interaction of a single user with the recommender system (also referred to as the platform or the algorithm).
The user and the items are assumed to have $d$-dimensional features that are time-invariant.
The user's feature vector is denoted by $\gt$, and is unknown to the bandit algorithm, whereas item features are assumed to be known. 
We refer to items by their feature vector $\x$. 
The user's utility for any item $\x$ is taken to be $\util{\x}$, the inner product of the user and item feature vectors.
Without loss of generality, we may assume $\norm{\gt} = 1$; utilities can be scaled by scaling the item features.

At each round $t$, the platform is presented with a set of candidate items $\X_t$ (called the context set).
The context set should be interpreted as the result of a filtering out of the vast universe of items, based on a search query or other factors such as location, time, etc.
While, in principle, the context set size could vary over time, we take it to be constant over the execution horizon. 
We assume $\X_t$ is stochastic; in particular, each item in $\X_t$ is drawn i.i.d. from a $d$-dimensional distribution $\dist$.
We make two restrictions on $\dist$; First, we assume any $\x \sim \dist$ satisfies $\norm{\x} \leq r$ almost surely. 
Second, we assume that the matrix $\Sigma = \mathbb{E}_{\x, \y \sim \dist \text{ i.i.d.}} [(\x - \y)(\x - \y)^\top]$ is invertible, that is, it is positive definite.

The platform picks an item $\x_t$ from $\X_t$ and recommends it to the user. 
This recommendation is based on the platform's estimate of the user's feature vectors. 
After the user has consumed $\x_t$, the platform seeks a comparison query from the user.
Let $\H_t$ denote the set of all items consumed up to time $t$: $\H_t = \{\x_1, \ldots, \x_{t-1}\}$; we refer to $\H_t$ as the \textit{history}. 
The platform then picks $\y_t$ from $\H_t$ and asks the user to compare $\x_t$ and $\y_t$. 
We do not impose any other restriction on the choice of $\y_t$; in particular, an item can be chosen for comparison arbitrarily often. 
Note the difference from the concurrent CDB model, where $\y_t$ is chosen from $\X_t$. Consequently, our notion of regret, the key performance measure of any bandit algorithm, also differs from the concurrent case.
Let $\x^*_t = \argmax_{\x \in \X_t} \util{\x}$ denote the best item at time $t$. Then $r_t = \util{\x^*_t} - \util{\x_t}$, the instantaneous regret at time $t$, measures the suboptimality of the recommended item $\x_t$. 
The concurrent CDB model, on the other hand, measures the regret in terms of both $\x_t$ and $\y_t$.
The cumulative regret, $R_T = \sum_{t \in [T]} r_t$, measures the performance of the algorithm over a time horizon $T$. 
Our goal is to design an algorithm where $R_T$ scales as $O(\sqrt{T})$ with high probability (ignoring logarithmic terms).

We assume that the user makes comparison decisions according to a simple probabilistic model called the linear stochastic transitivity model \cite{bengs2022stochastic}, which we describe below. 
Denote the outcome of the user's comparison by the binary variable $o_t$; we say $o_t = 1$ if the user picks $\x_t$ (denoted by the event $\x_t \succ \y_t$) and $o_t = 0$ otherwise. 
When asked to compare $\x_t$ and $\y_t$, the user picks $\x_t$ with probability $\link(\util{\x_t} - \util{\y_t})$;  more succinctly, $\P(o_t = 1) = F(\util{\z_t})$, with $\z_t$ denoting $\x_t - \y_t$.
The function $\link(\cdot)$ is called the \textit{link function}, and determines the noise in the comparisons.
This choice model conveys the intuition that users are more likely to pick an item with larger utility.
A prototypical example of $\link(\cdot)$ is the sigmoid function: $\sigma(x) = 1/(1 + \exp(-x))$. This special case corresponds to the popular Bradley-Terry choice model \cite{saha2021optimal}.
More generally, $F(\cdot)$ is a smooth, strictly increasing function from $\R \rightarrow (0,1)$, and satisfies $F(u) + F(-u) = 1$.

We assume that the link function is known to the algorithm.
Knowing the link function allows the algorithm to calculate the maximum likelihood estimate $\mle_t$ of $\gt$, given the collected data up to time $t$. 
Let $\D_t$ denote the dataset $\{\z_i, o_i\}_{i \in [t-1]}$ (recall $\z_i = \x_i - \y_i)$. 
Then the maximum likelihood estimator over $\D_t$ is obtained by solving $ \sum_{i \in [t-1]} (\link(\langle \z_i, \theta \rangle) - o_i)\z_i = 0$.
In practice, this can be solved using Newton's method or any variant of it.
For the special case of $\link$ being the sigmoid function, this estimation problem is identical to the logistic regression problem over the dataset $\D_t$.
This estimation step is common in many generalized linear bandit and CDB algorithms (see \cite{bengs2022stochastic, li2017provably}).

\paragraph{Algorithm.}\label{sec:alg} 
In this section, we present a simple and efficient algorithm for the history-constrained contextual duelling bandits model, called \algname  
(see pseudocode of Algorithm \ref{alg:bandit} below). The name reflects the fact that once an item is consumed and its regret paid for, it can be reused many times to pose a comparison query to the user; the algorithm exploits this fact.

\vspace{-.2cm}
\begin{algorithm}
\caption{Regret Once, Ask Many (\algname)}\label{alg:bandit}
\begin{algorithmic}[1]
\State \textbf{Input:} time horizon $T$, pure exploration horizon $\tau$
\State \textbf{Initialization:} $\H_1 = \emptyset$, $\D_1 = \emptyset$ 
\For {$t=1, \ldots ,\tau$}
\State Pick $\x_t \in \X_t$ uniformly at random and recommend it to the user
\State Set $\y_t \leftarrow \x_{t-1}$ and ask the user to compare $\x_t$ to $\y_t$ 
\State Update $\D_{t+1} \leftarrow \D_{t} \cup \{\x_t - \y_t, o_t\}$ based on the comparison outcome $o_t$
\State Update $\H_{t+1} \leftarrow \H_{t} \cup \{\x_t\}$
\EndFor
\State Set
$V_{\tau+1} \leftarrow \sum_{i=1}^\tau (\mathbf{x}_i - \mathbf{y}_i) (\mathbf{x}_i - \mathbf{y}_i)^{\top}$
\For{$t=\tau+1, \tau+2, \ldots, T$}
\State Calculate the MLE  $\mle_t$ over $\D_{t}$
\State Select $\mathbf{x}_t \leftarrow \argmax_{\x \in \X_t}\utilml{\x}$ and recommend it to the user
\State Set $\y_t \leftarrow \argmax_{\y \in \H_t} \normvt{\x_t - \y}$ and ask the user to compare $\x_t$ to $\y_t$ 
\State Update $\D_{t+1} \leftarrow \D_{t} \cup \{\x_t - \y_t, o_t\}$ based on the comparison outcome $o_t$
\State Update $V_{t+1} \leftarrow V_t+(\x_t - \y_t) (\x_t - \y_t)^{\top}$
\State Update $\H_{t+1}\leftarrow \H_{t} \cup \{\x_t\}$
\EndFor
\end{algorithmic}
\end{algorithm}

The algorithm begins with a pure exploration phase that lasts $\tau$ rounds. In this phase, the platform picks an item uniformly at random from the context set and recommends it to the user. Then, it asks the user to compare this with the item recommended just before. This exploration phase serves two purposes. First, it provides enough data such that $\mle_\tau$ is a good-enough estimate of $\gt$; this helps the algorithm to perform well in subsequent rounds. This aspect is fairly standard in the literature \cite{saha2021optimal, bengs2022stochastic, li2017provably}. 
Second, it populates the history $\H_\tau$ with a diverse set of items. We show in the next section, ensuring sufficient diversity is crucial for formulating informative queries in subsequent rounds. 
While some minimum degree of exploration is essential to satisfy these two criteria, too large of an exploration can lead to poor regret. Thus, $\tau$ must be carefully chosen based on the system parameters. In the derivation of our regret bounds, we specify a suitable value of $\tau$.

The main phase of the algorithm is described in lines 12 and 13 of Algorithm \ref{alg:bandit}. The algorithm picks $\x_t$, the new item to recommend to the user, by optimizing the {\em estimated utility} $\utilml{\x}$ over the arms. This strategy can be interpreted as being greedy or purely exploitative, as it picks the item that is most likely to yield the lowest regret. It is useful to contrast this step with the corresponding step in \colstim, the algorithm proposed in \cite{bengs2022stochastic}, where the objective being optimized is $\utilml{\x} + \epsilon\normvt{\x}$ ($\epsilon$ being a bounded random variable). This additional term reflects the uncertainty in the utility of item $\x$. The weighted sum is thus a careful balance of exploration and exploitation. Indeed, in the UCB algorithm for generalized linear bandits, a near-identical objective is optimised \cite{li2017provably}.

Before interpreting the choice of the comparison arm $\y_t$, it is instructive to draw an analogy with linear regression in order to better understand the process of estimating $\gt$. Upon asking a comparison query $(\x_t, \y_t)$ and recording the corresponding outcome $o_t$, the algorithm essentially `probes' the unknown variable $\gt$ in the direction of $\z_t=\x_t - \y_t$. 
For this reason, we refer to $\z_t$ as the probing vector. Repeated probes in the same direction leads to a better estimate of $\gt$ in that direction. The matrix $V_t$, being the sum of the outer products of the probing vectors so far (see lines 9 and 15 of Algorithm \ref{alg:bandit}), summarizes the combined effect of all probes on $\gt$. Put differently, $V_t^{-1}$ reflects the uncertainty in the estimate $\mle_t$. The weighted norm $\normvt{\cdot}$ gives a higher emphasis to vectors aligned along directions of larger uncertainty.

With the above interpretation in mind, it is easy to interpret the choice of $\y_t$ as a pure exploration step. By optimising $\normvt{\x_t - \y}$, we choose $\y_t$ that probes along the direction of maximum uncertainty. Once again, it is instructive to compare this action against the corresponding action in \colstim, in which, $\y_t$ is chosen by maximizing a weighted sum of $\normvt{\x_t - \y}$ and $\utilml{\y - \x_t}$, reflecting a balance of exploration and exploitation.
We have the freedom to optimize purely for exploration while choosing $\y_t$ because we do not (directly) suffer any regret in the choice of $\y_t$. Thus, it is natural to choose $\y_t$ in a manner that would lead to the best improvement in the estimate $\mle_t$. This, in turn, contributes to reducing the cumulative regret.

An important consequence of our algorithm's design is that, with the exception of the length of the initial exploration period, there is no hyperparameter to be optimised. Unlike our method, most bandit algorithms (not just \colstim) optimize a weighted sum of the exploitation and exploration terms. Theoretically, the level of exploration is dictated by the concentration bounds on $\error$. Often, one observes that a lower degree of exploration works better empirically than what is suggested in theory; this is because the concentration bounds can be quite loose. Thus, the relative weight of exploration to exploitation ({\em e.g.}, $C_{\text{thresh}}, c_1$ in \colstim) is often a hyper-parameter that needs to be optimized. No such hyper-parameter tuning is required in our algorithm.

\section{Main Result}\label{sec:main_result}

In this section, we state and prove our main result (Theorem \ref{thm:main}): a high probability bound on the regret of our algorithm \algname. Our proof also prescribes the length of the initial exploration time $\tau$. Before stating the theorem, we recall the parameters $d, r,$ and $\Sigma$ from Section \ref{sec:model}. 
Also recall that we have assumed $\Sigma$ to be positive definite, which implies $\lambda_{\min}(\Sigma)$ is strictly positive. Define $\kappa$ to be:
\begin{equation}\label{eq:kappa}
    \kappa = \inf_{\z, \boldsymbol{\theta}: \norm{\z} \,\leq\, 2r, \norm{\boldsymbol{\theta} - \gt} \,\leq\, 1} \link'(\langle \z, \boldsymbol{\theta} \rangle) 
\end{equation}

\begin{theorem}[Regret Bound for \algname]\label{thm:main}
    Let $\delta \in (0,1)$ be given. Suppose \algname~ is run with $\tau = O(r^2/{\lambda_{\min}(\Sigma)}) \log(d/\delta)$. Then, with probability at least $1 - 3\delta$, \algname~  suffers a cumulative regret of 
    \begin{equation*}
        R_T = O\left(\frac{r}{\kappa\sqrt{\lambda_{\min}(\Sigma)}}d\sqrt{T}\log\left(\frac{T}{d\delta}\right)\right)
    \end{equation*}
    where the $O(\cdot)$ notation hides absolute constants.
\end{theorem}

It is useful to compare this result with the $O(\kappa^{-1}d\sqrt{T}\log\left(T/{d\delta}\right)$ regret bounds for similar algorithms in the generalised linear bandit model \cite{li2017provably} and the concurrent contextual duelling bandit setting \cite{saha2021optimal, bengs2022stochastic}. The results in the literature are proven under the assumption $r = 1$; thus, our bounds include an extra factor of $1/\sqrt{\lambda_{\min}(\Sigma)}$. As we had discussed in Section \ref{sec:intro}, the fact that $\x_t$ and $\y_t$ are not chosen from the same set poses some difficulties in the regret analysis; this extra factor is a manifestation of this difficulty (see Section \ref{sec:rich_history} below). In many practical scenarios, this term may be reasonably small. Indeed, if we assume that the distribution $\dist$ is the uniform distribution over the unit ball in $d$ dimensions, then $1/\sqrt{\lambda_{\min}(\Sigma)}$ scales as $\sqrt{d}$. Our simulations demonstrate that in practice, the regret is smaller than what the bounds suggest (see Section \ref{sec:experimental_results}).

The proof of Theorem \ref{thm:main} rests on three key lemmas. Lemmas \ref{lem:impo2} and \ref{lem:impo1} are results that we borrow from prior work on contextual duelling bandits \cite{saha2021optimal}; they give us bounds on $\Vert \error \Vert_{V_t}$ and $\normvt{\x^*_t - \x_t}$ respectively. These bounds, in turn, sourced from a paper on generalized linear bandits \cite{li2017provably} (\cite{saha2021optimal} makes the important connection between contextual duelling bandits and generalized linear bandits). In fact, the origin of both of these lemmas can be traced to the seminal work on linear bandits \cite{saha2021optimal}. Both these lemmas assume $\lambda_{\min}(V_{\tau+1})\geq 1$. Below, we sketch out an argument of why this is likely to hold. 

Recall that $V_{\tau+1} = \sum_{i \in \tau}\z_i\z_i^\top$, where $\z_i = \x_i - \y_i$. Because of our comparison strategy in the exploration phase, $\z_2, \z_4, \ldots$ are i.i.d. random vectors. Thus, we expect the empirical average $(1/(\tau/2))\sum_{i \in \tau/2} \z_{2i}\z_{2i}^\top$ to be approximately the same as the statistical average, $\Sigma$. We have assumed that $\lambda_{\min}(\Sigma)$ is strictly positive. Thus, for sufficiently large $\tau$, we can expect that $\sum_{i \in \tau/2} \z_{2i}\z_{2i}^\top$ has all eigenvalues larger than one. The additional component in $V_{\tau + 1}$ is a positive semidefinite matrix, and therefore $V_t$'s eigenvalues are at least as large. In the Appendix, we prove a sufficient number of steps $\tau$ needed so that $\lambda_{\min}(V_{\tau+1})\geq 1$ holds with with probability $1-\delta$. Note that Lemma \ref{lem:impo2} holds with high probability, while Lemma \ref{lem:impo1} gives a deterministic bound. Also note the subtle differences between Lemmas \ref{lem:impo2} and \ref{lem:impo1} and the corresponding lemmas in \cite{saha2021optimal}, arising from the fact that we consider context vectors of norm $r$ rather than one, and general linear choice model rather than Plackett-Luce. 

\begin{lemma}\label{lem:impo2}
     Suppose for some $\tau \in [T]$, $\lambda_{\min}(V_{\tau+1})\geq 1$. Fix $\delta \in (0, 1)$. Then, with probability at least $1-\delta$, 
     \vspace{-.3cm}
     \begin{align*}
         \forall \ t > \tau, \Vert \error \Vert_{V_t} \leq \alpha, \ \ \text{where }\ {\small\alpha = \frac{1}{\kappa}\sqrt{\frac{d}{2}\log \left( 1+\frac{2t}{d} \right)+\log \left(\frac{1}{\delta}\right)}
         }.
     \end{align*}
\end{lemma}

\begin{lemma}\label{lem:impo1}
     Suppose for some $\tau \in [T]$, $\lambda_{\min}(V_{\tau+1})\geq 1$. Then, for all $t > 0$,
     \begin{align*}
         \sum_{i=\tau+1}^{t+\tau}\normvt{\x^*_t - \x_t} \leq \sqrt{2td\log \left(\frac{4r^2\tau+t}{d}\right) }.
     \end{align*}
\end{lemma}

\subsection{Critical Ratio and Rich History}\label{sec:rich_history}
In addition to the above lemmas, the proof of Theorem \ref{thm:main} crucially uses the following bound: $\normvt{\x_t - \x^*_t} \leq \beta \normvt{\x_t - \y_t}$. A smaller value of $\beta$ leads to a tighter regret bound. For this reason, we refer to the ratio $\normvt{\x_t - \x^*_t} / \normvt{\x_t - \y_t}$ as the \textit{critical ratio}. Observe that if we were to run \algname~in the concurrent CDB model, {\em i.e.,} $\y_t$ were to be chosen from $\X_t$, then $\beta = 1$ would suffice (a similar observation is used by \cite{saha2021optimal} in their proof of the regret bound). However, since $\y_t$ is chosen from $\H_t$, controlling the critical ratio is more challenging as it depends on the entire trajectory up to time $t$ via $V_t^{-1}$ and $\H_t$. Our main insight is that a suitable bound may be found by considering the worst-case scenario over $\x_t, \x^*_t$, and $V_t^{-1}$. The following definition captures this notion.

\begin{definition}[Rich history]\label{def:perfect}
$\mathcal{H}_t$, the history at time $t$, is said to be $\beta$-rich if, for any $\x, \x'$ such that $\norm{\x} \leq r$ and $\norm{\x'} \leq r$, and for any positive definite matrix $A$, the following bound holds:
\begin{align*}
    \normA{\x - \x'} \leq \beta  \max_{\y \in \H_t}\normA{\x - \y},
\end{align*}
\end{definition}
The following lemma shows that after a suitable number of steps of pure exploration, the history remains rich forever after (with high probability).

\begin{lemma}\label{lem:perfect}
   Let $\delta \in (0,1)$ be given. There exists a universal constant $C > 0$ such that if we run \algname~with
   \begin{equation*}
       \tau = \frac{Cr^2}{\lambda_{\min}(\Sigma)} \log(d/\delta),
   \end{equation*}
   then with probability at least $1 - \delta$, the history $\mathcal{H}_{t}$ at any time $t > \tau$ is $\frac{8r}{\sqrt{\lambda_{\min}(\Sigma)}}$-rich.
\end{lemma}

The main intuition behind this result is that all the `richness' of the history comes from the initial pure exploration rounds. Indeed, with sufficient exploration, we are likely to have items with features spanning `all directions'. Given this diversity of items in the history, for any $\x$ and $A$, it is possible to choose a $\y$ from the history that is (reasonably) well-aligned with the direction of the principal eigenvector of $A$ (corresponding to the largest eigenvalue). This gives a lower bound on the denominator of the term $\max_{\y \in \H_t}\normA{\x - \y}$. The remaining steps to get an upper bound on the critical ratio are not hard. The full derivation of this result is given in the appendix.

\subsection{Proof of Main Theorem}
The proof of Theorem \ref{thm:main} proceeds in a manner that is quite similar to the corresponding proofs in \cite{saha2021optimal} and \cite{li2017provably}. We begin by obtaining a bound on the instantaneous regret $r_t$, following which we bound the cumulative regret.

\begin{proof}[Proof of Theorem \ref{thm:main}]
    Throughout this proof, we assume that the inequalities of Lemmas \ref{lem:impo2}, \ref{lem:impo1}, and \ref{lem:perfect} hold. Lemmas \ref{lem:impo2} and \ref{lem:perfect} hold individually with probability $1-\delta$; in addition, the condition $\lambda_{\min}(V_{\tau+1}) \geq 1$ needed for Lemmas \ref{lem:impo2} and \ref{lem:impo1} holds with probability $1-\delta$. Thus, by the union bound, all three of them hold with probability $1-3\delta$.
    
    Recall that, for any $t$, $r_t = \util{\x^*_t} - \util{\x_t}$. For any $t > \tau$, we obtain:
    \begin{align*}
        r_t &= \util{\x^*_t} - \util{\x_t} = \utilml{\x^*_t - \x_t} + \langle \x^*_t - \x_t, \error \rangle \stackrel{(i)}{\leq} \langle \x^*_t - \x_t, \error \rangle \\
        &\stackrel{(ii)}{\leq} \Vert \error \Vert_{V_t} \normvt{\x^*_t - \x_t} \stackrel{(iii)}{\leq} \alpha \normvt{\x^*_t - \x_t} \stackrel{(iv)}{\leq} \alpha\beta \normvt{\x_t - \y_t} 
    \end{align*}
    where inequality $(i)$ holds because of the choice of $\x_t$ in \algname~ implies $\utilml{\x^*_t - \x_t} \leq 0$, $(ii)$ holds because of Cauchy-Schwartz inequality, $(iii)$ uses Lemma \ref{lem:impo2}, and $(iv)$ uses Lemma \ref{lem:perfect}. 

    Further, for $t \leq \tau$, the instantaneous regret can be bounded by using the Cauchy-Schwarz inequality, the triangle inequality, and the fact that $\norm{\x} \leq r$ and $\norm{\gt} = 1$.
    \begin{align*}
        r_t = \util{\x^*_t} - \util{\x_t} = \util{\x^*_t - \x_t} \leq \norm{\x^*_t - \x_t}\norm{\gt} \leq 2r
    \end{align*}
    Using these bounds on the instantaneous regret, we get
    \begin{align*} 
        R_T = \sum_{t=1}^\tau r_t +\sum_{t=\tau+1}^T r_t  &\stackrel{(i)}{\leq} 2r\tau + \alpha \beta \sum_{t=\tau+1}^T \normvt{\x_t - \y_t} \stackrel{(ii)}{\leq} 2r\tau+\alpha \beta \sqrt{2dT\log\left(\frac{4r^2\tau+T}{d}\right)},
    \end{align*}
    where (i) uses the bounds on $r_t$ derived above and (ii) follows from Lemma \ref{lem:impo1}.
    
    Finally, plugging in the values of $\alpha$, $\beta$ and $\tau$ from Lemmas \ref{lem:impo2} and \ref{lem:perfect}  respectively, we get:
    \begin{align*} 
    R_T & \leq \frac{cr^2}{\lambda_{\min}(\Sigma)} \log(d/\delta)+\frac{8r}{\kappa \sqrt{\lambda_{\min}(\Sigma)}} \sqrt{\frac{d}{2}\log \left( 1+\frac{2t}{d} \right)+\log \left(\frac{1}{\delta}\right)} \sqrt{2dT\log\left(\frac{4r^2\tau+T}{d}\right)}.
\end{align*}
Simplifying this expression yields $R_T = O\left((r/\kappa\sqrt{\lambda_{\min}(\Sigma)})d\sqrt{T}\log\left(T/d\delta\right)\right)$ as claimed in Theorem \ref{thm:main}.
\end{proof}

\section{Experimental Results}\label{sec:experimental_results}

In this section, we present experimental results illustrating the behaviour of our algorithm \algname~on synthetic data. The default parameters used in the default important parameters of the experiments are written below. The ambient dimension $d$ is set to be five. We generate both $\gt$ and the context vectors uniformly on the unit ball of radius one. Thus, $r = 1$ and $1/\sqrt{\lambda_{\min}(\Sigma)}$ is $\sqrt{d}$. We generate $k = 1000$ context vectors, sampled i.i.d. at each step. We simulate the user via a Plackett-Luce choice model; in other words, the link function $F(\cdot)$ is taken to be the sigmoid function $\sigma(x) = 1/(1 + \exp(-x))$. With this choice, we are able to use SciKitLearn's inbuilt method for Logistic Regression to calculate $\hat{\theta}_t$. The initial exploration period is set to be $10d$, and the total time horizon $T$ is set to $1000$. Note that these are default values; in the experiments that follow, we explore the performance of the algorithm by varying one parameter at a time. 

For each experiment, we plots three quantities: the cumulative regret $R_T$, the error of the MLE estimate, $\norm{\error}$, and the critical ratio ${\normvt{\x_t - \x^*_t}}/{\normvt{\x_t - \y_t}}$. While the regret is the main performance metric of bandit algorithms, the estimation error and critical ratio are key quantities that influence the regret of our algorithm. Each plot is averaged over $n = 100$ independent runs. The solid curves shown reflect the mean over $n$ runs and the shaded region shows a 95\% confidence interval on the estimate of the mean ($\pm2\sigma/\sqrt{n}$; $\sigma$ being the empirical standard deviation over the $n$ runs). In addition, the critical ratio plots are smoothened by a moving average window of length ten. All experiments were run on locally on a MacBook with the M1 Pro chip and 16GB RAM. For all experiments, a single run (over $T = 1000$ timesteps) took approximately one second.

Our first experiment, Figure~\ref{fig:dim_variation}, examines the variation of these three quantities with the underlying dimension $d$. We choose five different values of $d$: $\{2, 4, 6, 8, 10\}$. The rest of the parameters are chosen as per the default values; in particular, $\tau$ grows linearly with $d$. We plot the regret, error, and the critical ratio as a function of time in Figures \ref{fig:dim_subfig1}, \ref{fig:dim_subfig2}, and \ref{fig:dim_subfig3} respectively. As expected, the regret increases linearly in the initial pure exploration phase, followed by a sublinear increase in the main explore-exploit phase. We observe that the regret increases with the dimension, roughly linearly. The error in the estimate, $\norm{\mle - \gt}$, decreases gracefully over time, but increases with $d$. Interestingly, we observe that the critical ratio is much smaller in practice than our theoretical bound. However, as we predicted by theory, the ratio increases with the dimension $d$. It also decreases slightly over time.
\begin{figure}[htbp]
  \centering
  \begin{subfigure}[b]{0.32\textwidth}
    \centering
    \includegraphics[width=\linewidth]{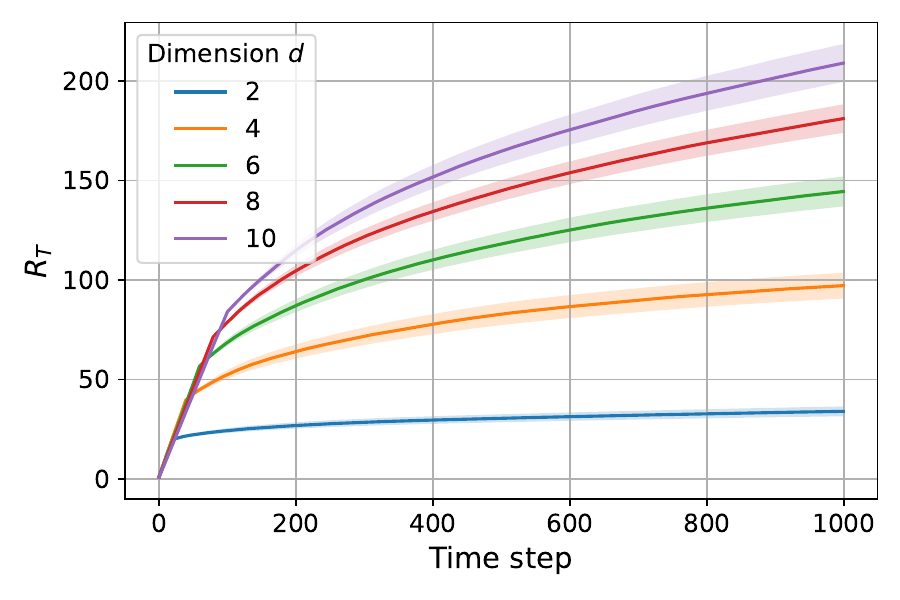}
    \caption{Cumulative Regret}
    \label{fig:dim_subfig1}
  \end{subfigure}
  \hfill
  \begin{subfigure}[b]{0.32\textwidth}
    \centering
    \includegraphics[width=\linewidth]{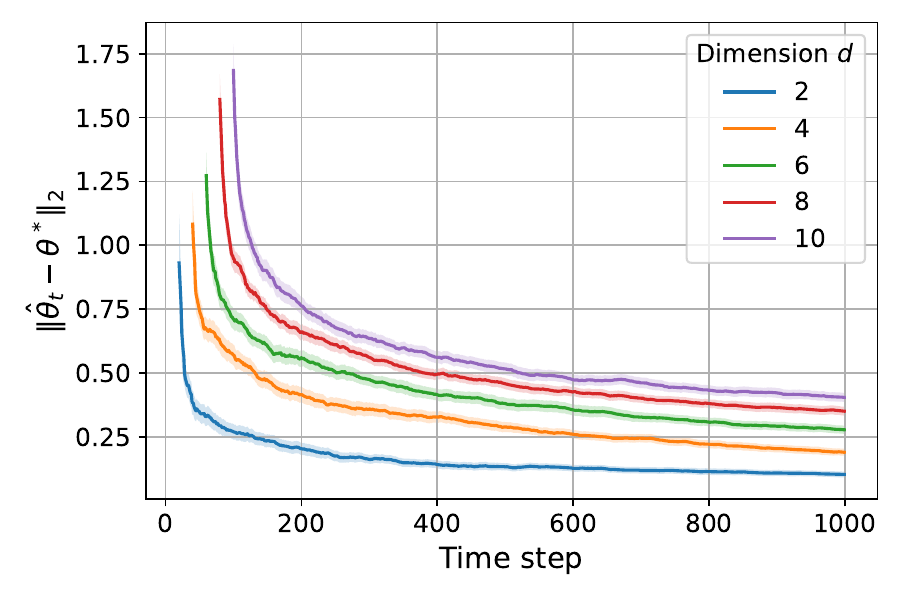}
    \caption{Error in ML estimate}
    \label{fig:dim_subfig2}
  \end{subfigure}
  \hfill
  \begin{subfigure}[b]{0.32\textwidth}
    \centering
    \includegraphics[width=\linewidth]{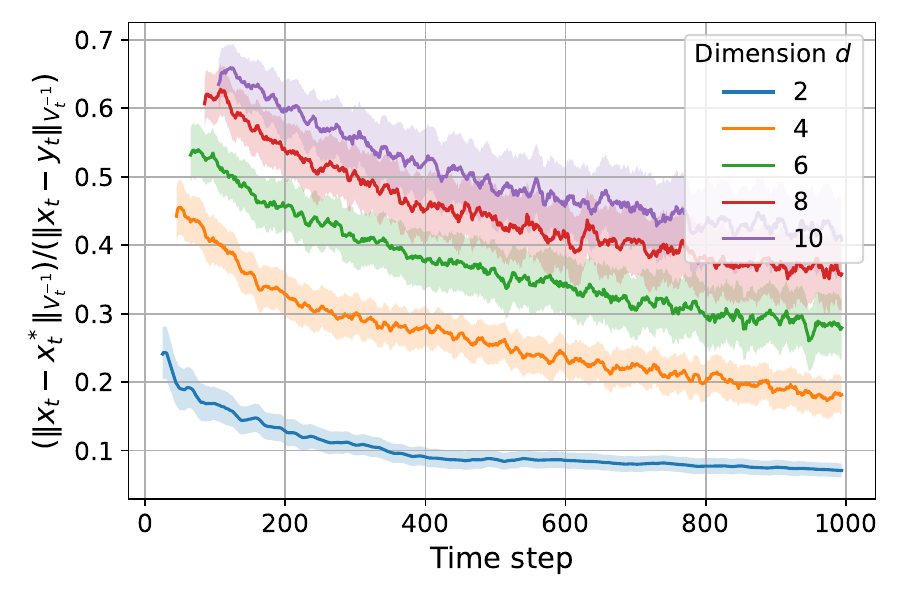}
    \caption{Critical Ratio}
    \label{fig:dim_subfig3}
  \end{subfigure}
  \caption{Variation of regret, error, and the key ratio as a function of dimension $d$.}
  \label{fig:dim_variation}
\end{figure}
\vspace{-.2cm}

Our next plot, Figure \ref{fig:tau_variation} examines the role of the initial pure exploration phase. We vary $\tau$ among the values $\{0, 25, 50, 75, 100\}$, keeping all other parameters to their default values listed above. In the case where $\tau = 0$, we add a small regularizer parameter of $\lambda = 0.1$ while calculating the maximum likelihood estimator. (We experimented with different values of $\lambda$ and found this to be the best). We observe in Figure \ref{fig:tau_subfig1} that as $\tau$ increases from 25 to 100, the regret progressively worsens. This is because the extra initial exploration affects neither the error in $\mle$, nor the critical ratio. Thus, the extra exploration penalty is not compensated for (see Figures \ref{fig:tau_subfig2}, \ref{fig:tau_subfig3}). However, when $\tau = 0$, the regret is large. This is because both the error and the critical ratio are significantly larger for this case. Thus, the experiment concludes that a small amount of initial exploration is optimal for $\algname$.
\begin{figure}[htbp]
  \centering
  \begin{subfigure}[b]{0.32\textwidth}
    \centering
    \includegraphics[width=\linewidth]{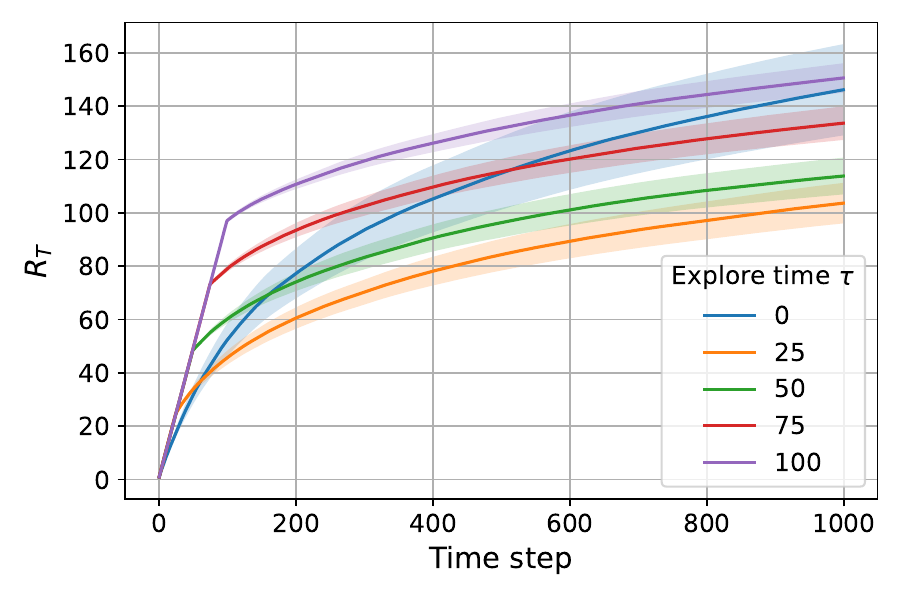}
    \caption{Cumulative Regret $R_T$}
    \label{fig:tau_subfig1}
  \end{subfigure}
  \hfill
  \begin{subfigure}[b]{0.32\textwidth}
    \centering
    \includegraphics[width=\linewidth]{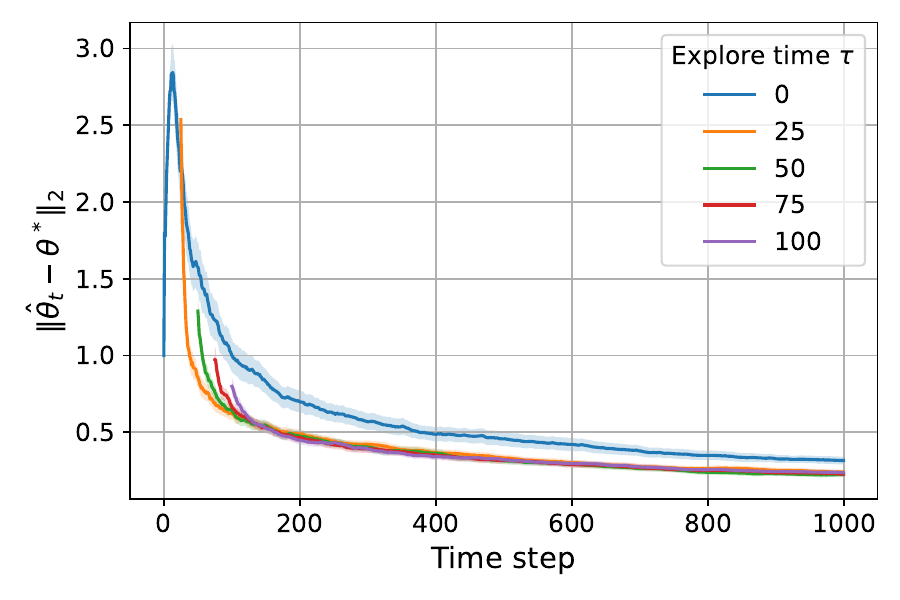}
    \caption{$\norm{\mle - \gt}$}
    \label{fig:tau_subfig2}
  \end{subfigure}
  \hfill
  \begin{subfigure}[b]{0.32\textwidth}
    \centering
    \includegraphics[width=\linewidth]{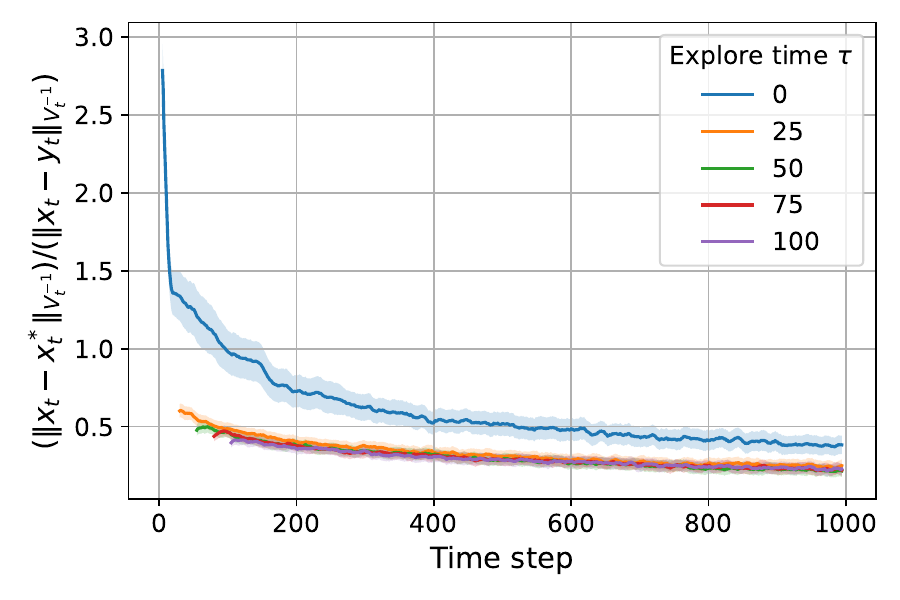}
    \caption{${\normvt{\x_t - \x^*_t}}/{\normvt{\x_t - \y_t}}$}
    \label{fig:tau_subfig3}
  \end{subfigure}
  \caption{Variation of regret, error, and the key ratio as a function of exploration time $\tau$.}
  \label{fig:tau_variation}
\end{figure}

Our final experiment (Figure \ref{fig:comparison}) compares \algname, an algorithm designed for the history-constrained CDB model, with \colstim, the state-of-the-art algorithm for concurrent CDBs \cite{bengs2022stochastic}. All parameters are kept at the default values listed above. For \colstim, we experiment with a few values of the hyper-parameters. $c_1$ controls the exploration-exploitation tradeoff in the choice of $\y$, and $c_2$ (called $C_{\text{thresh}}$ in \cite{bengs2022stochastic}) controls a similar tradeoff in the choice of $\x$. We vary these hyperparameters, choosing $c_1 \in \{1, 10\}$ and $c_2 \in \{0.1, 1\}$. Recall that for \algname, there is no such hyper-parameter. We plot the regret in terms of $\x$ alone for \algname~ and the average regret in terms of $\x$ and $\y$ for \colstim. This is indeed a fair comparison, for one counts a unit of regret for every item recommended and for every comparison made. Figure \ref{fig:comparison_subfig1} shows that \algname~outperforms \colstim~over all values of the hyperparameters. The performance can be better understood through Figures \ref{fig:comparison_subfig2} and \ref{fig:comparison_subfig3}. On the one hand, for $c_1 = 1$ (low exploration), we observe that \colstim does not learn $\gt$ quickly, leading to large regret. On the other hand, for $c_1 = 10$ (large exploration), the exploratory nature of $\y$ leads to a large regret, despite $\mle_t$ converging quickly to $\gt$. A deeper understanding can be obtained by casting the concurrent model in the history-constrained setting. Assume the context set remains unchanged for two successive time steps (say $t$ and $t+1/2$). \colstim~ first recommends $\x_t$ and then $\y_t$ (paying regret for both), and then compares the latest item ($\y_t$) to \textit{the one consumed just before} ($\x_t$). \algname, on the other hand, recommends $\x_t$ at both successive time steps (paying twice the regret), but compares the latest item to \textit{any suitable item in the past}. This experiment clearly demonstrates the benefit of recycling items from the user's consumption history.

\begin{figure}[htbp]
  \centering
  \begin{subfigure}[b]{0.32\textwidth}
    \centering
    \includegraphics[width=\linewidth]{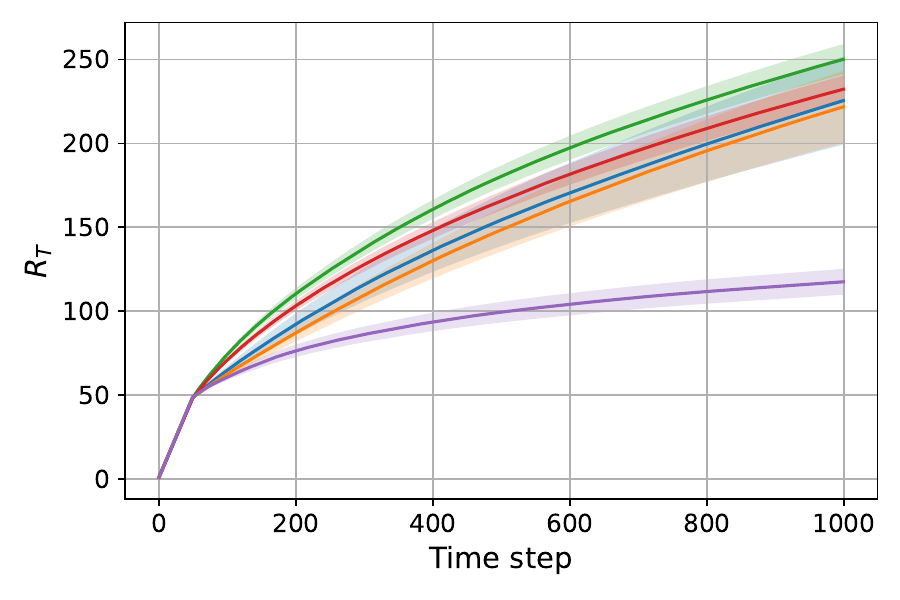}
    \caption{Cumulative Regret $R_T$}
    \label{fig:comparison_subfig1}
  \end{subfigure}
  \hfill
  \begin{subfigure}[b]{0.32\textwidth}
    \centering
    \includegraphics[width=\linewidth]{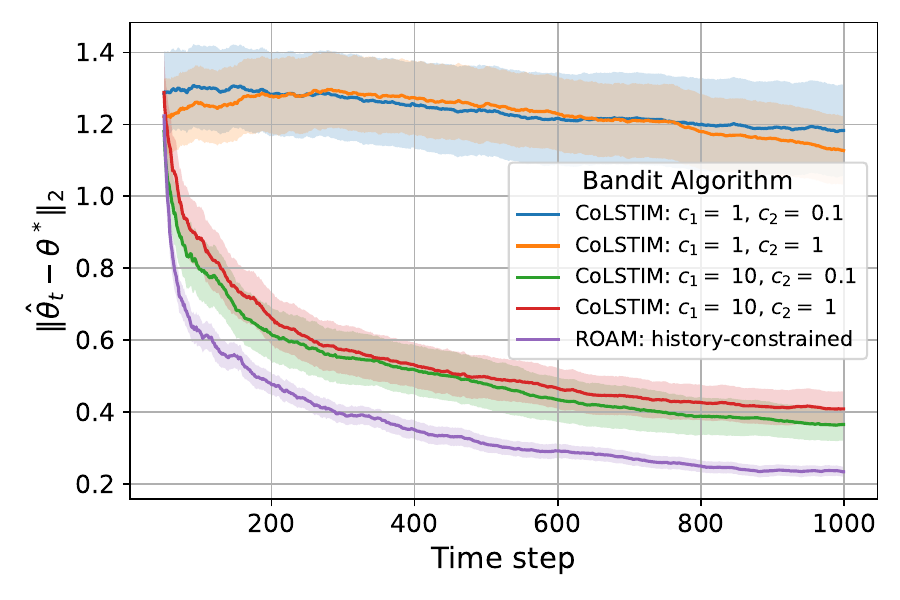}
    \caption{$\norm{\mle - \gt}$}
    \label{fig:comparison_subfig2}
  \end{subfigure}
  \hfill
  \begin{subfigure}[b]{0.32\textwidth}
    \centering
    \includegraphics[width=\linewidth]{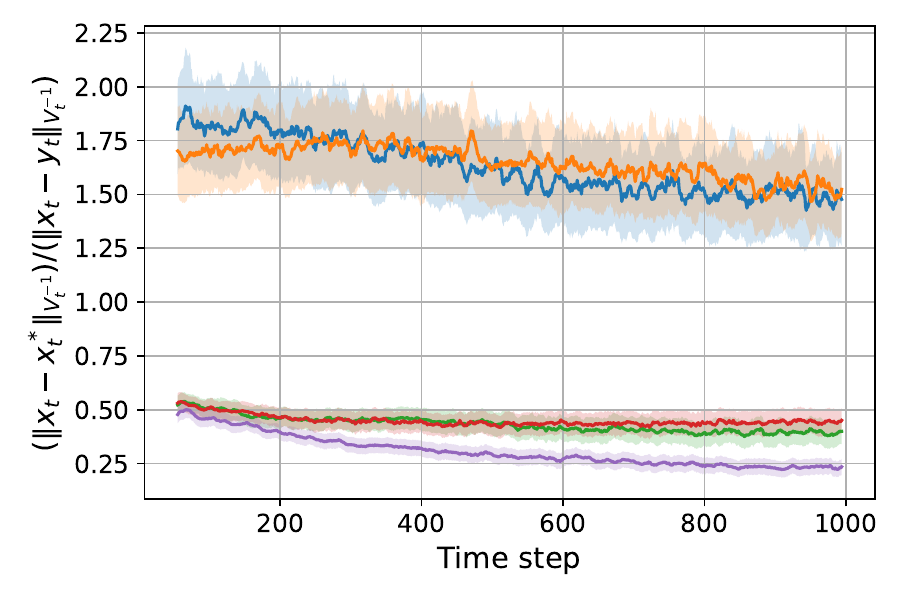}
    \caption{${\normvt{\x_t - \x^*_t}}/{\normvt{\x_t - \y_t}}$}
    \label{fig:comparison_subfig3}
  \end{subfigure}
  \caption{A comparison of \algname~ with \colstim, keeping all parameters identical.}
  \label{fig:comparison}
\end{figure}

\section{Discussion}\label{sec:discussion}

\paragraph{Summary} In this work, we propose a contextual bandit framework that learns from explicit comparisons between consumed items. The key idea is to reuse previously recommended items—incurring no new regret—for future comparisons. Our algorithm, \algname, exploits this flexibility to accumulate a rich history within a short exploration time. This rich history enables informative queries subsequently, yielding $O(\sqrt{T})$ regret. Our simulations show that \algname~outperforms the state-of-the-art in concurrent CDBs. Our work opens up several avenues of future research, which we discuss below.

\paragraph{Model Interpretation} We assume that $\gt$ encodes a single user's fixed preferences, and item features are known and time-invariant. These assumptions ensure a consistent notion of a user's history, useful for selecting past items for comparison. In alternative contextual bandit interpretations, item features may vary over time ({\em e.g.}, by incorporating user metadata), and $\gt$ may consequently reflect global feature weights. While our analysis depends on the particular viewpoint we adopt, the broader insight, that comparing with items consumed in the past is valuable, is likely to hold more generally. Formalizing this intuition in other contextual bandit models remains an open direction.

\paragraph{Handling Nonlinearity} Like prior work on CDBs \cite{saha2021optimal, bengs2022stochastic}, our  regret bounds scale with $1/\kappa$, which can grow exponentially with $r$. This is a result of the analysis technique, which uses a uniform lower bound on $\link'(\cdot)$ (see Lemma \ref{lem:impo2}). Recent work on generalized linear bandits avoids this dependence on $\kappa$ by more refined handling of the nonlinearity \cite{dong2019performance, faury2020improved, abeille2021instance}. Extending such techniques to the CDB setting could improve regret bounds, but remains an open challenge.

\paragraph{Comparison Horizon} Our algorithm selects comparison items from the entire history, often reusing those from the exploration phase. In practice, it may be preferable to compare with recently consumed items. A caveat is that once the user's preferences are reasonably estimated, such items are likely to be similar. In terms of the richness of the history, this could be severely detrimental (see Section \ref{sec:rich_history}). However, comparisons among similar queries could actually prove useful to resolve fine-grained rankings among items. Alternate algorithms, potentially guided by criteria like the information ratio \cite{dong2019performance}, may prove to be effective in this scenario.

\paragraph{Choice Model} Like prior work, we assume that choice probabilities depend only on utility differences. This idealization may fail in practice: some item pairs may not be comparable, noise may be temporally correlated, and preferences may be shaped by context or prior experiences (e.g., anchoring) \cite{tversky1972elimination, tversky1974judgment}. Capturing such effects requires richer choice models. Developing practical bandit algorithms under more realistic user behaviour remains an important direction for future work.

\paragraph{Holistic Cost} In our model, we treat item consumption as costly (incurring regret) and comparisons as essentially free. However, taken to an extreme, this would suggest querying repeatedly with every new item recommended, which is impractical. A more realistic model would assign costs to queries and optimize the total cost: regret plus query burden.

%%%%%%%%%%%%%%%%%%%%%%%%%%%%%%%%%%%%%%%%%%%%%%%%%%%%%%%%%%%%

\bibliographystyle{unsrt}
\bibliography{NeurIPS_2025/clean_references}

%%%%%%%%%%%%%%%%%%%%%%%%%%%%%%%%%%%%%%%%%%%%%%%%%%%%%%%%%%%%
\newpage
\appendix
\section{Remaining Proofs}

\subsection{Basics of Symmetric Matrices}\label{sec:basics_symmatrix}

Let $\pd$ denote the set of positive definite matrices of dimension $d$. Any $A \in \pd$ can be expressed in terms of its eigenvalues and eigenvectors as follows:
\begin{equation}
    A = \sum_{i = 1}^d \lambda_i(A) v_i(A)v_i(A)^\top,
\end{equation}
where $\lambda_1(A) \geq \lambda_2(A) \ldots \geq \lambda_d(A) \geq 0$ are the eigenvalues of $A$. $v_1(A), \ldots, v_d(A)$ are the corresponding eigenvectors, forming an orthonormal basis of $\R^d$. In particular, $\norm{v_i(A)} = 1 \ \forall \ i, \ \forall \ A$. When the matrix $A$ is clear from context, we use the simpler notation $\lambda_i$ and $v_i$.

For any $A \in \pd$, the matrix-induced norm $\Vert \cdot \Vert_A$ is defined as:
\begin{equation}
    \normA{z} = \sqrt{z^\top Az} \ \forall \ z, \ \forall \ A
\end{equation}
By the eigen-decomposition of the matrix, it follows that
\begin{equation}\label{eq:lowerbound_by_maxeig}
    \normA{z}^2 \geq \lambda_1(A) \langle z, v_1(A) \rangle^2 \ \forall \ z, \ \forall \ A
\end{equation}
We can also derive the following upper bound:
\begin{equation}\label{eq:upperbound_by_maxeig}
    \normA{z}^2 \leq \lambda_1(A) \norm{z}^2 \ \forall \ z, \ \forall \ A
\end{equation}

For any matrix $A$, let $\norm{A}$ denote the induced $\ell_2$ norm of the matrix. If $A \in \pd$, we have the identity
\begin{equation}\label{eq:def_symmatrix_norm}
    \norm{A} = \sup_{v: \norm{v} = 1} v^\top A v = \lambda_1(A)
\end{equation}

Let $\Sigma$ be an arbitrary positive definite matrix and $z$ an arbitrary vector (both in $d$ dimensions). Let $y = \Sigma^{-1/2}z$ and $B = \Sigma^{1/2} A\Sigma^{1/2}$. Then $\normA{z} = \Vert y \Vert_B$. Further, 
\begin{equation}\label{eq:whitening_transformation}
    \lambda_1(B) = \lambda_1(\Sigma^{1/2} A\Sigma^{1/2}) = \sup_{v: \norm{v} = 1} v^\top \Sigma^{1/2} A\Sigma^{1/2} v \geq \left(\lambda_{\min}(\Sigma^{1/2})\right)^2 \sup_{v: \norm{v} = 1} v^\top A v = \lambda_1(A) \lambda_{\min}(\Sigma)
\end{equation}

\subsection{Basic Concentration Results}
\begin{definition}[Isotropic Distribution]
    A random vector $z \in \R^d$ is said to satisfy an isotropic distribution if $\E[zz^\top] = I_d$; here, $I_d$ refers to the identity matrix in $d$ dimensions.
\end{definition}

\begin{lemma}[Key Concentration Inequality]\label{lem:concentration}
    Let $z_1, z_2, \ldots$ be an i.i.d. sequence of random vectors from an isotropic distribution, satisfying the bound $\norm{z} \leq r$ almost surely. Let $\epsilon \in (0,1)$ and $\delta \in (0,1)$ be given. Suppose $\tau \geq c(r^2/\epsilon^2)\log(d/\delta)$, where $c > 0$ is a universal constant. Then, with probability at least $1-\delta$,
    \begin{equation*}
        \normbig{\frac{1}{\tau} \sum_{t = 1}^\tau z_tz_t^\top - I_d} \leq \epsilon
    \end{equation*}
\end{lemma}
\begin{proof}
    This lemma follows in a straightforward fashion from a matrix concentration result proven in \cite{vershynin2012introduction}, namely Theorem 5.41. We reproduce this result here, in our notation. This result states that there exists an absolute constant $c' > 0$ such that the following statement holds. For any $s > 0$, with probability at least $1 - 2d\exp(-c's^2)$, 
    \begin{equation*}
        \normbig{\frac{1}{\tau} \sum_{t = 1}^\tau z_tz_t^\top - I_d} \leq  \max(\varepsilon, \varepsilon^2)  =: \epsilon \ \text{where } \varepsilon = s \frac{r}{\sqrt{\tau}}.
    \end{equation*}
    Since we have chosen $\epsilon$ to be less than one, $\max(\varepsilon, \varepsilon^2)$ equals $\varepsilon$. It remains to deduce a value of $\tau$ as a function of $\epsilon$ and $\delta$.
    
    Setting $2d\exp(-c's^2)$ to $\delta$ gives us $s = O(\sqrt{\log(d/\delta)})$. Setting $sr/\sqrt{\tau}$ to $\epsilon$ gives us $\tau = O(r^2/\epsilon^2) \log(d/\delta)$. The constant hidden in the $O(\cdot)$ notation is a universal constant.
\end{proof}

\begin{lemma}[Existence of a Good Vector]\label{lem:good_vector}
    Let $z_1, z_2, \ldots$ be an i.i.d. sequence of random vectors from an isotropic distribution, satisfying the bound $\norm{z} \leq r$ almost surely. Let $\epsilon \in (0,1)$ and $\delta \in (0,1)$ be given. There exists a universal constant $c > 0$ such that if $\tau \geq c(r^2/\epsilon^2)\log(d/\delta)$, then with probability at least $1-\delta$,
    \begin{equation*}
        \inf_{v: \norm{v} = 1} \ \max_{t \in [\tau]} \langle z_t, v \rangle^2 \geq 1 - \epsilon 
    \end{equation*}
\end{lemma}
\begin{proof}
    This result follows easily from Lemma \ref{lem:concentration}. We outline the steps below. Suppose the concentration result holds (which happens with probability $1-\delta$). Then
    \begin{align*}
        \normbig{\frac{1}{\tau} \sum_{t = 1}^\tau z_tz_t^\top - I_d} 
        &\leq \epsilon \\
        \Rightarrow \forall \ v: \norm{v} = 1, \ 
        \left\vert v^\top \left(\frac{1}{\tau} \sum_{t = 1}^\tau z_tz_t^\top - I_d\right) v \right\vert 
        &\leq \epsilon \quad (\text{by } \eqref{eq:def_symmatrix_norm}) \\
        \Rightarrow \inf_{v: \norm{v} = 1} \ 
        \left\vert\frac{1}{\tau} \sum_{t = 1}^\tau \langle v, z_t \rangle^2 - \norm{v}^2\right\vert 
        &\leq \epsilon  \\
        \Rightarrow \inf_{v: \norm{v} = 1}  \ 
        \frac{1}{\tau} \sum_{t = 1}^\tau \langle v, z_t \rangle^2  
        &\geq 1-\epsilon \quad  (\because \norm{v} = 1) \\
        \Rightarrow \inf_{v: \norm{v} = 1}  \ 
        \max_{t \in [\tau]} \langle v, z_t \rangle^2  
        &\geq 1-\epsilon  \quad (\because \text{maximum is larger than average}) 
    \end{align*}
\end{proof}

\begin{lemma}[Lower Bound on Matrix Norm]\label{lem:lwrbd_matrixnorm}
    Let $z_1, z_2, \ldots$ be an i.i.d. sequence of random vectors from an isotropic distribution, satisfying the bound $\norm{z} \leq r$ almost surely. Let $\epsilon \in (0,1)$ and $\delta \in (0,1)$ be given. Suppose $\tau \geq c(r^2/\epsilon^2)\log(d/\delta)$, where $c$ is some universal constant. Then, with probability at least $1-\delta$,
    \begin{equation*}
        \forall \ {A \in \pd},  \ \max_{t \in [\tau]} \Vert z_t \Vert_A^2 \geq (1 - \epsilon) \lambda_{1}(A)
    \end{equation*}
\end{lemma}
\begin{proof}
    This result follows easily from Lemma \ref{lem:good_vector} and the basic inequalities concerning positive semidefinite matrices presented in Section \ref{sec:basics_symmatrix}. Starting from \eqref{eq:lowerbound_by_maxeig}, we get:
    \begin{align*}
        \forall \ A \in \pd, \ \forall \ z \in \R^d, \Vert z \Vert_A^2 
        &\geq \lambda_1(A) \langle z, v_1(A) \rangle^2 
        \  \\
        \Rightarrow \forall \ {A \in \pd}, \forall \  {t \in [\tau]}, \Vert z_t \Vert_A^2 
        &\geq \lambda_1(A) \langle z_t, v_1(A) \rangle^2 
    \end{align*}
    By definition, $\norm{v_1(A)} = 1$ for all $A \in \pd$. Therefore,
    By Lemma \ref{lem:good_vector}, with probability at least $1 - \delta$,
    \begin{align*}
        \forall \ {A \in \pd}, \max_{t \in [\tau]} \langle v_1(A), z_t \rangle^2  
        &\geq 1-\epsilon
    \end{align*}
    Putting these two inequalities together, we get that with probability at least $1-\delta$,
    \begin{align*}
        \forall \ {A \in \pd}, \ \max_{t \in [\tau]} \Vert z_t \Vert_A^2 
        &\geq \max_{t \in [\tau]} \lambda_1(A) \langle z_t, v_1(A) \rangle^2 \\
        &\geq (1-\epsilon)\lambda_1(A).
    \end{align*}
\end{proof}

The next lemma extends this result to the case of nonisotropic random vectors.

\begin{lemma}[Lower Bound  for Nonisotropic Vectors]\label{lem:lwrbd_history}
    Let $z_1, z_2, \ldots$ be an i.i.d. sequence of random vectors, with a distribution such that $\E[zz^\top] = \Sigma$ is invertible and the bound $\norm{z} \leq r$ holds almost surely. Let $\epsilon \in (0,1)$ and $\delta \in (0,1)$ be given. Suppose $\tau \geq c(r^2/\lambda_{\min}(\Sigma)\epsilon^2)\log(d/\delta)$, where $c$ is some universal constant. Then, with probability at least $1-\delta$,
    \begin{equation*}
        \forall \ {A \in \pd},  \ \max_{t \in [\tau]} \Vert z_t \Vert_A^2 \geq (1 - \epsilon) {\lambda_{1}(A)}\lambda_{\min}(\Sigma)
    \end{equation*}
\end{lemma}
\begin{proof}
    Define $w_i = \Sigma^{-1/2}z_i$. Then $w_1, w_2, \ldots$ is an i.i.d. sequence of isotropic random vectors satisfying the bound $\norm{w} \leq \norm{\Sigma^{-1/2}}r$, which equals $r/\sqrt{\lambda_{\min}(\Sigma)}$. Invoking Lemma \ref{lem:lwrbd_matrixnorm}, with $\tau = c(r^2/\lambda_{\min}(\Sigma)\epsilon^2)\log(d/\delta)$, we see that with probability at least $1-\delta$,
    \begin{equation*}
        \forall \ {A \in \pd},  \ \max_{t \in [\tau]} \Vert w_t \Vert_A^2 \geq (1 - \epsilon) \lambda_{1}(A)
    \end{equation*}
    The desired result follows by noting that $\Vert z_t \Vert_A^2$ = $\Vert w_t \Vert_B^2$, where $B = \Sigma^{1/2}A\Sigma^{1/2}$, and \eqref{eq:whitening_transformation}.
\end{proof}

\subsection{Proof of Rich History Condition}
\begin{lemma}[Triangle Inequality]\label{lem:triangle_ineq}
    For any $\tau$, for any $\x \in \R^d$, and for any $A \in \pd$,
    \begin{equation*}
        \max_{\y \in \H_{2\tau + 1}}\normA{\x - \y} \geq (1/2) \max_{t \in [\tau]} \normA{\x_{2t - 1} - \x_{2t}}
    \end{equation*}
\end{lemma}
\begin{proof}
    The result follows from a straightforward inequality of triangle inequality. Recall, by definition, $\H_{2\tau + 1} = \{\x_1, \ldots, \x_{2\tau}\}$. It follows that
    \begin{align*}
        \max_{\y \in \H_{2\tau + 1}}\normA{\x - \y} &= \max_{t \in [2\tau]} \normA{\x - \x_t} \\
        &= \max_{t \in [\tau]} \left(\max \{\normA{\x - \x_{2t-1}}, \normA{\x - \x_{2t}} \} \right) \\
        &\geq \max_{t \in [\tau]} (\normA{\x - \x_{2t-1}} + \normA{\x - \x_{2t}})/2 \\
        &\geq \max_{t \in [\tau]} (\normA{\x_{2t - 1} - \x_{2t}})/2
    \end{align*}
    The last step uses the triangle inequality with  respect to the norm $\normA{\cdot}$.
\end{proof}

Let $\x, \x'$ be i.i.d. random vectors with distribution $\dist$. Let $\Sigma = \E[(\x - \x')(\x - \x')^\top]$. Let $\mathcal{B}(r)$ denote the ball of radius $r$, i.e., $\mathcal{B}(r) = \{\x: \norm{\x} \leq r\}$. Using this notation, we state the main result of this section.
\begin{lemma}[Main Inequality]\label{lem:main_ineq}
    Let $\delta \in (0, 1)$ be given. Suppose the initial exploration phase of \algname~ is run for $\tau$ rounds, where $\tau \geq c(r^2/\lambda_{\min}(\Sigma))\log(d/\delta)$; $c$ being a universal constant. Then, with probability at least $1 - \delta$,
    \begin{align*}
        \forall \ \x, \x' \in \mathcal{B}(r), \ \forall A \ \in \pd, \ \normA{\x - \x'} \leq \frac{8r}{\sqrt{\lambda_{\min}({\Sigma})}} \sup_{\y \in \H_{2\tau + 1}}\normA{\x - \y}.
    \end{align*}
\end{lemma}
\begin{proof}
    The proof consists of deriving an upper bound for the term on the left hand side ($\normA{\x - \x'}$) and a lower bound for the right hand side ($\sup_{\y \in \H_{2\tau + 1}}\normA{\x - \y}$). 

    Combining \eqref{eq:upperbound_by_maxeig} with the fact that both $\x$ and $\x'$ have norm at most $r$, we get the following upper bound:
    \begin{align}\label{eq:main_ineq1}
        \normA{\x - \x'} \leq \sqrt{\lambda_1(A)} \norm{\x - \x'} \leq 2r\sqrt{\lambda_1(A)}
    \end{align}

    By Lemma \ref{lem:triangle_ineq}, we get that
    \begin{equation}\label{eq:main_ineq2}
        \max_{\y \in \H_{2\tau + 1}}\normA{\x - \y} \geq (1/2) \max_{t \in [\tau]} \normA{\x_{2t - 1} - \x_{2t}}
    \end{equation}

    Denote $\x_{2t - 1} - \x_{2t}$ by $\z_t$. Then $\z_1, \ldots, \z_\tau$ are i.i.d. random vectors satisfying $\norm{\z} \leq r$ almost surely and $\E[\z\z^\top] = \Sigma$. Invoking Lemma \ref{lem:lwrbd_history} with $\epsilon = 3/4$, we get that if $\tau \geq O(r^2/\lambda_{\min}(\Sigma))\log(d/\delta)$, then with probability at least $1-\delta$,
    \begin{equation}\label{eq:main_ineq3}
        \max_{t \in [\tau]} \normA{\x_{2t - 1} - \x_{2t}} \geq \frac{1}{2}\sqrt{\lambda_1(A)\lambda_{\min}(\Sigma)}
    \end{equation}

    Combining \eqref{eq:main_ineq1}, \eqref{eq:main_ineq2}, and \eqref{eq:main_ineq3}, the stated result follows.
\end{proof}

Lemma \ref{lem:main_ineq} states that the history at time $2\tau + 1$ is $\beta$-rich with high probability, for $\beta = 8r/\sqrt{\lambda_{\min}(\Sigma)}$.
Lemma \ref{lem:perfect} follows from Lemma \ref{lem:main_ineq} by an appropriate change of notation from $\tau$ to $2\tau$ and noting that once the history at time $\tau+1$ is $\beta$ rich, the history at all subsequent times is also $\beta$ rich. This is because for all $t \geq \tau + 1$, $\H_{\tau + 1} \subseteq \H_t$, which implies $\sup_{\y \in \H_{\tau + 1}}\normA{\x - \y}$ $\leq $ $\sup_{\y \in \H_{t}}\normA{\x - \y}$.

\subsection{Result on $V_\tau$}
\newcommand{\B}{\mathbb{B}}
\newcommand{\w}{\mathbf{w}}
In this document, we show that, with high probability, $\lambda_{\min}(V_{\tau + 1}) \geq 1$. This result is used in our proof of the main theorem. Before we get to this result, we prove a simple result regarding the relation between eigenvalues of positive definite matrices $A$ and $B$ that are congruent (see Lemma \ref{lem:congruent_matrix}). Let $\B$ denote the unit ball in $d$-dimensions, {\em i.e.}, $\B = \{\x: \norm{\x} = 1\}$. Recall that for any $A \in \pd$,
    \begin{align*}
        \lambda_{\max}(A) = \max_{\x \in \B} \x^\top A \x\, ; \quad \lambda_{\min}(A) = \min_{\x \in \B} \x^\top A \x
    \end{align*}
Using these definitions of the eigenvalues, we prove the following lemma.
\begin{lemma}\label{lem:congruent_matrix}
    Let $A, \Sigma \in \pd$ be given. Let $B \triangleq \Sigma^{1/2}A\Sigma^{1/2}$. Then $B \in \pd$. Further, the following inequalities hold:
    \begin{align}
        \lambda_{\max}(B) &\geq \lambda_{\min}(\Sigma) \lambda_{\max}(A) \label{eq:lambdamax_bound} \\
        \lambda_{\min}(B) &\geq \lambda_{\min}(\Sigma) \lambda_{\min}(A) \label{eq:lambdamin_bound}
    \end{align}
\end{lemma}
\begin{proof}
    The claim that $B$ is positive definite is easily verified from the definition. We proceed to prove \eqref{eq:lambdamin_bound} first. Let $\underline{\x}$ denote the eigenvector of $B$ corresponding to its smallest eigenvalue (of norm one). Let $\underline{\y}$ be the vector of norm one aligned along $\Sigma^{1/2}\underline{\x}$, that is,
    \begin{equation*}
        \underline{\y} = \Sigma^{1/2}\underline{\x}/\normbig{\Sigma^{1/2}\underline{\x}}.
    \end{equation*}
    Thus, $\underline{\y} \in \B$. Furthermore, we have the inequality 
    \begin{equation*}
        \normbig{\Sigma^{1/2}\underline{\x}} = \sqrt{\underline{\x}\Sigma\underline{\x}} \geq \sqrt{\lambda_{\min}(\Sigma)} \quad (\because \ \Sigma \in \pd \text{ and } \underline{\x} \in \B)
    \end{equation*}
    We now have all the ingredients to prove \eqref{eq:lambdamin_bound}.
    \begin{align*}
        \lambda_{\min}(B) = \underline{\x}^\top B \underline{\x} = \underline{\x}^\top \Sigma^{1/2}A\Sigma^{1/2} \underline{\x} = \normbig{\Sigma^{1/2}\underline{\x}}^2 \underline{\y}^\top A \underline{\y} \\\geq \lambda_{\min}(\Sigma)\underline{\y}^\top A\underline{\y} \geq \lambda_{\min}(\Sigma) \inf_{\y \in \B} \y^\top A \y = \lambda_{\min}(\Sigma)\lambda_{\min}(A)
    \end{align*}

    The proof of \eqref{eq:lambdamax_bound} follows similar steps, but with a careful reordering of the arguments. Let $\overline{\y}$ denote the eigenvector of $A$ corresponding to its largest eigenvalue (of norm one). Let $\overline{\x}$ be the vector of norm one aligned along $\Sigma^{-1/2}\overline{\y}$, that is,
    \begin{equation*}
        \overline{\x} = \Sigma^{-1/2}\overline{\y}/\normbig{\Sigma^{-1/2}\overline{\y}}.
    \end{equation*}
    Thus, $\overline{\x} \in \B$. Also note that
    \begin{equation*}
        \normbig{\Sigma^{-1/2}\overline{\y}} = \sqrt{\overline{\y}^\top \Sigma^{-1} \overline{\y}} \leq \sqrt{\lambda_{\max}(\Sigma^{-1})} = 1/\sqrt{\lambda_{\min}(\Sigma)}
    \end{equation*}
    Finally, note that 
    \begin{equation*}
        B = \Sigma^{1/2}A\Sigma^{1/2} \ \Leftrightarrow \ A = \Sigma^{-1/2}B\Sigma^{-1/2}
    \end{equation*}
    Putting these equations together, we get:
    \begin{align*}
        \lambda_{\max}(A) = \overline{\y}^\top A \overline{\y} = \overline{\y}^\top \Sigma^{-1/2}B\Sigma^{-1/2} \overline{\y} = \normbig{\Sigma^{-1/2}\overline{\y}}^2 \overline{\x}^\top B \overline{\x} \\\leq \lambda_{\min}^{-1}(\Sigma)\overline{\x}^\top B \overline{\x} \leq \lambda_{\min}^{-1}(\Sigma) \sup_{\x \in \B} \x^\top B \x = \lambda_{\min}^{-1}(\Sigma)\lambda_{\max}(B)
    \end{align*}
    This proves \eqref{eq:lambdamax_bound}, and also provides a more detailed justification for \eqref{eq:whitening_transformation} (which is used in the proof of Lemma \ref{lem:lwrbd_history}.
\end{proof}

\begin{lemma}
    Let $\delta \in (0, 1)$ be given. Suppose the initial exploration phase of \algname~ is run for $\tau$ rounds, where $\tau \geq c(r^2/\lambda_{\min}(\Sigma))\log(d/\delta)$; $c$ being a universal constant. Then, with probability at least $1 - \delta$, $\lambda_{\min}(V_{\tau + 1}) \geq 1$.
\end{lemma}
\begin{proof}
    Recall that, in the pure exploration phase, $\y_t = \x_{t-1}$. Therefore,
    \begin{equation*}
        V_{\tau + 1} = \sum_{t = 1}^{\tau} (\x_t - \y_t)(\x_t - \y_t)^\top = \sum_{t = 1}^{\tau} (\x_t - \x_{t-1})(\x_t - \x_{t-1})^\top 
    \end{equation*}
    We can express $V_{\tau + 1}$ as the sum of two matrices, $V^{\text{even}}_{\tau + 1} + V^{\text{odd}}_{\tau + 1}$, where:
    \begin{equation*}
        V^{\text{even}}_{\tau + 1} = \sum_{t = 1}^{\tau/2} (\x_{2t} - \x_{2t-1})(\x_{2t} - \x_{2t-1})^\top\, ; \quad V^{\text{odd}}_{\tau + 1} = \sum_{t = 1}^{\tau/2} (\x_{2t-1} - \x_{2t-2})(\x_{2t-1} - \x_{2t-2})^\top
    \end{equation*}
    Recall that $\x_1, \x_2, \ldots$ are i.i.d. random vectors with distribution $\dist$. Denoting $\x_{2t} - \x_{2t-1}$ by $\z_{2t}$, we observe that $\z_2, \z_4, \ldots$ are i.i.d. random vectors satisfying $\E[\z\z^\top] = \Sigma$ and $\norm{\z} \leq 2r$ almost surely. 
    
    Let $\w_i \triangleq \Sigma^{-1/2} \z_{2i}$. Then $\w_1, \w_2, \ldots$ are i.i.d. random vectors satisfying $\E[\w\w^T] = I$ and $\norm{\z} \leq 2r\sqrt{\norm{\Sigma^{-1/2}}}$ almost surely. We also know that $\sqrt{\norm{\Sigma^{-1/2}}} = 1/\sqrt{\lambda_{\min}(\Sigma)}$, which gives us $\norm{\z} \leq 2r/\sqrt{\lambda_{\min}(\Sigma)}$. 
    
    Define the matrix $$U_{\tau} = \left(\frac{2}{\tau}\right) \sum_{i \in [\tau/2]} \w_i\w_i^\top$$
    Observe that $\w_i = \Sigma^{-1/2} \z_{2i}$ implies $\z_{2i} = \Sigma^{1/2}\w_i$. Thus,
    \begin{align*}
        V^{\text{even}}_{\tau + 1} = \sum_{t = 1}^{\tau/2} \z_{2t} \z_{2t}^\top = \left(\frac{\tau}{2}\right)\Sigma^{1/2} U_{\tau} \Sigma^{1/2}
    \end{align*}
    Using Lemma \ref{lem:congruent_matrix} (in particular, \eqref{eq:lambdamin_bound}), we get
    \begin{align*}
        \lambda_{\min}(V^{\text{even}}_{\tau + 1}) = \left(\frac{\tau}{2}\right)\lambda_{\min}(\Sigma^{1/2} U_{\tau} \Sigma^{1/2}) \geq \left(\frac{\tau }{2}\lambda_{\min}(\Sigma)\right) \lambda_{\min}(U_{\tau}) 
    \end{align*}
    Applying Lemma \ref{lem:concentration} to $U_\tau$ with $\epsilon = 1/2$, we conclude that with probability at least $1-\delta$,
    \begin{align*}
        \lambda_{\min}(U_{\tau}) \geq 1/2 \ \Rightarrow \lambda_{\min}(V^{\text{even}}_{\tau + 1}) \geq \frac{\tau }{4}\lambda_{\min}(\Sigma) \Rightarrow \lambda_{\min}(V_{\tau + 1}) \geq \frac{\tau }{4}\lambda_{\min}(\Sigma)
    \end{align*}
    The last step follows from the fact that adding the positive semidefinite matrix $V^{\text{odd}}_{\tau + 1}$ to $V^{\text{even}}_{\tau + 1}$ can only raise its minimum eigenvalue. Finally, we know that $\tau \geq c(r^2/\lambda_{\min}(\Sigma))\log(d/\delta)$. Choosing $c$ large enough, we get that $\tau/4 \geq 1/\lambda_{\min}(\Sigma)$, which implies $\lambda_{\min}(V_{\tau + 1}) \geq 1$.
\end{proof}

%%%%%%%%%%%%%%%%%%%%%%%%%%%%%%%%%%%%%%%%%%%%%%%%%%%%%%%%%%%%

\end{document}